\pdfoutput=1

\documentclass[10pt,twocolumn,letterpaper]{article}

\usepackage[pagenumbers]{cvpr}

\usepackage{amsthm}
\usepackage{bm}
\newtheorem{proposition}{Proposition}
\newtheorem{assumption}{Assumption}

\definecolor{signpostblue}{HTML}{9DC3E6}
\makeatletter
\g@addto@macro\normalsize{%
 \setlength{\abovedisplayskip}{6pt plus 1pt minus 2pt}%
 \setlength{\belowdisplayskip}{6pt plus 1pt minus 2pt}%
 \setlength{\abovedisplayshortskip}{0pt plus 1pt}%
 \setlength{\belowdisplayshortskip}{3pt plus 1pt minus 1pt}%
}
\makeatother

\makeatletter
\renewcommand\paragraph{\@startsection{paragraph}{4}{\z@}{\z@}{-1em}{\normalfont\normalsize\bfseries}}
\makeatother

\ifdefined\pdfsuppressptexinfo
\fi

\newcommand{\signpost}[2]{\par\noindent\hspace{0.1em}{\color{signpostblue}\vrule width 1.2pt}\hspace{0.5em}%
  \parbox[t]{\dimexpr\linewidth-0.6em-1.2pt\relax}{\strut\textbf{#1.}\enspace #2\strut}\par\prevdepth=\dp\strutbox\relax\vspace{0.35\baselineskip}}

\definecolor{cvprblue}{rgb}{0.21,0.49,0.74}
\usepackage[pagebackref,breaklinks,colorlinks,allcolors=cvprblue]{hyperref}

\def\confName{CVPR}
\def\confYear{2027}

\title{MM-LeJEPA: Correlated Gaussian Regularization for Multimodal Joint-Embedding Predictive Architectures}

\author{Haowen Bai$^{1}$~~
Rui Gong$^{1}$~~
Meichen Liu$^{1}$~~
Ming-Xuan Weng$^{2}$~~
Bihan Wen$^{1}$~~
Kim-Hui Yap$^{1}$\thanks{Corresponding author.}~~
\\[1mm]
$^{1}$Nanyang Technological University\quad  \\
$^{2}$Hyundai Motor Group Innovation Center Singapore (HMGICS)\quad
\\[1mm]
\texttt{haowen.bai@hotmail.com}
}

\begin{document}
\maketitle
\begin{abstract}

LeJEPA shows that, when nothing is known about the downstream task, an isotropic Gaussian minimizes a bound on the bias of downstream probes, and enforces it with a lightweight regularizer. Its task prior, however, treats every embedding direction alike. Applied to a joint audio-visual embedding, the isotropic target penalizes the cross-modal correlation that shared content induces. We propose MM-LeJEPA, which extends the task prior to multimodal tasks and derives the embedding distribution it implies. We model downstream tasks as weighted mixtures of two readouts: one uses independent information from the two modalities, the other uses the same content from both. Minimizing the same bound, a weighted Fisher information, under a fixed total variance then gives a correlated Gaussian target: each modality stays isotropic, and audio and video are coupled by a single correlation $\rho$. A fixed whitening transform lets SIGReg enforce this target, so the objective keeps LeJEPA's two terms and adds one scalar, with no reconstruction decoder, EMA teacher, or contrastive negatives. On AudioSet and VGGSound, MM-LeJEPA is competitive with open-source audio-visual methods under both frozen probing and fine-tuning. Code will be released.
\end{abstract}
\section{Introduction}

\begin{figure}[t]
 \centering

 \includegraphics[width=\linewidth]{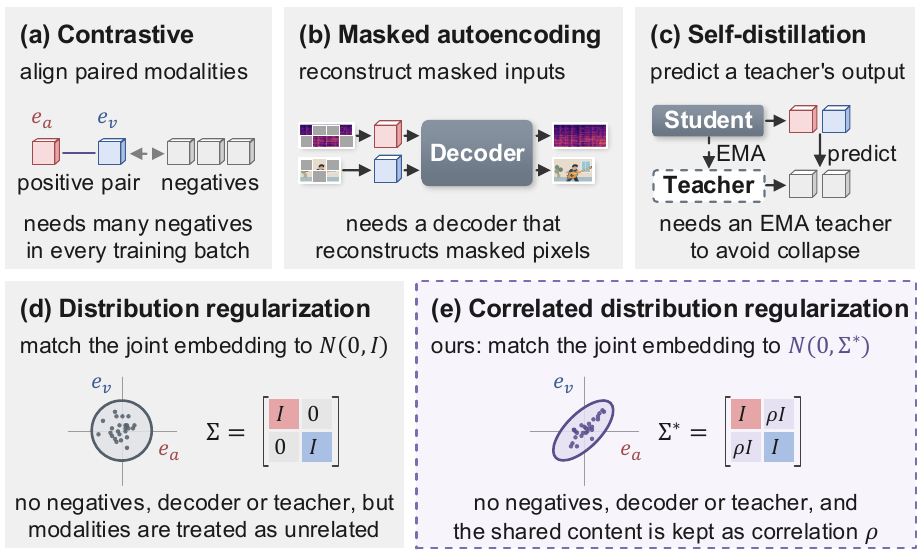}
 \vspace{-2em}
 \caption{{Self-supervised objectives for paired modalities. (a)--(c) need negatives, a decoder, or an EMA teacher; (d) distribution regularization needs none of these but treats the two modalities as unrelated; (e) MM-LeJEPA keeps their shared content as a correlation $\rho$ in the target distribution.}}
 \label{fig:teaser}
 \vspace{-1em}
\end{figure}

Self-supervised learning aims to learn representations from unlabeled data that transfer to unknown downstream tasks. Across unimodal and multimodal settings, common approaches include contrastive learning~\citep{simclr,moco}, self-distillation~\citep{byol,dino}, and masked reconstruction~\citep{mae,avmae}; each prevents collapse with an additional component, namely negatives, a decoder, or an exponential-moving-average (EMA) teacher (Fig.~\ref{fig:teaser}a--c). Audio-visual methods combine these strategies with cross-modal alignment~\citep{cavmae,mavil}, fine-grained synchronization~\citep{cavmaesync}, and dedicated fusion mechanisms~\citep{mbt}, which improve performance in downstream applications~\citep{li2023avmst,li2024stft,li2023mdst,li2025mdstpp,yang2025shmamba}. {These recipes work well, yet each relies on its own safeguard against collapse, and choices about masking, alignment losses and encoder sharing interact, so every new combination needs its own tuning}.

LeJEPA~\citep{lejepa} removes this burden for a single modality by asking what distribution the embeddings should follow. It keeps the principle of joint-embedding predictive architectures (JEPAs), predicting one view's representation from another~\citep{lecun2022path,ijepa,vjepa}, and replaces the usual collapse-prevention mechanism~\citep{ijepa,vicreg} with a derived target distribution. It shows that, when nothing is known in advance about the downstream task and the total variance is fixed to rule out trivial rescaling, an isotropic Gaussian minimizes a weighted Fisher information criterion that bounds the average prediction bias of downstream probes, nonparametric predictors trained on the embedding. It then uses a lightweight regularizer, Sketched Isotropic Gaussian Regularization (SIGReg), to pull the embeddings toward this distribution. The objective has two terms, prediction and regularization, and needs none of the extra components that the methods above add to prevent collapse (Fig.~\ref{fig:teaser}d). {Without the regularizer, the prediction loss would admit a trivial solution in which every view is mapped to the same constant embedding.}

LeJEPA was developed for a single modality, and its task prior treats all embedding directions alike because nothing is assumed about the task. In LeJEPA this prior is what fixes the target, and multimodal learning differs precisely in the structure of downstream tasks. With one encoder per modality, the standard design in audio-visual pretraining because it allows single-modality inference~\citep{cavmae,mavil,equiav}, the joint embedding splits into an audio block and a video block. Audio and video may observe the same event, so part of the content captured by one encoder is also captured by the other, and a downstream task can read this shared content from both blocks at once. Applying LeJEPA's isotropic Gaussian target directly to the concatenated audio-visual embedding requires the two blocks to be uncorrelated. It therefore penalizes the cross-modal correlation induced by shared content, although shared-content tasks need this correlation to read the same content from both blocks (Fig.~\ref{fig:teaser}d). Then, what should the target distribution be when part of the content is shared?

Our key observation is that, with two modalities, something is known about the downstream task after all: part of the content is shared, so a task may read the same content from both blocks. A dog's bark is heard and seen at once, so its identity can be read from either block; the colour of its collar is read from the video alone. The task prior should therefore remain free of preferences within each modality, as in LeJEPA, but allow the readouts of the two modalities to be coupled, {as a shared-content task requires}.

We distill this observation into MM-LeJEPA, a multimodal extension of LeJEPA. We model downstream tasks as weighted combinations of two components, one that reads independent information from the two modalities and one that reads the same content from both. With LeJEPA's no-preference assumptions kept within each modality, this yields a multimodal task prior. The optimal embedding distribution under this prior is a correlated Gaussian (Fig.~\ref{fig:teaser}e): each modality remains an isotropic Gaussian, while the audio and video embeddings are correlated through one coefficient $\rho$. A fixed whitening transform maps this target to the standard Gaussian, so SIGReg enforces it unchanged. The target reduces to LeJEPA's when shared content is absent, and $\rho$ reads as the share of shared content in the task prior; since the prior's parameters are not known before pretraining, we treat it as a hyperparameter. Our contributions are:

\begin{itemize}
 \item A multimodal task prior that adds one distinction to LeJEPA's, between readouts of shared content and readouts of independent information, and otherwise keeps its no-preference assumptions.

 \item {The optimal embedding distribution under this prior, a correlated Gaussian with a single parameter $\rho$, obtained by minimizing the weighted Fisher information functional under a fixed total variance.}
 \item {Comparisons with open-source methods on public audio-visual benchmarks show competitive probing results and fine-tuning performance.}
\end{itemize}
\section{Related Work}
\begin{figure*}[!t]
 \centering
 \includegraphics[width=\linewidth]{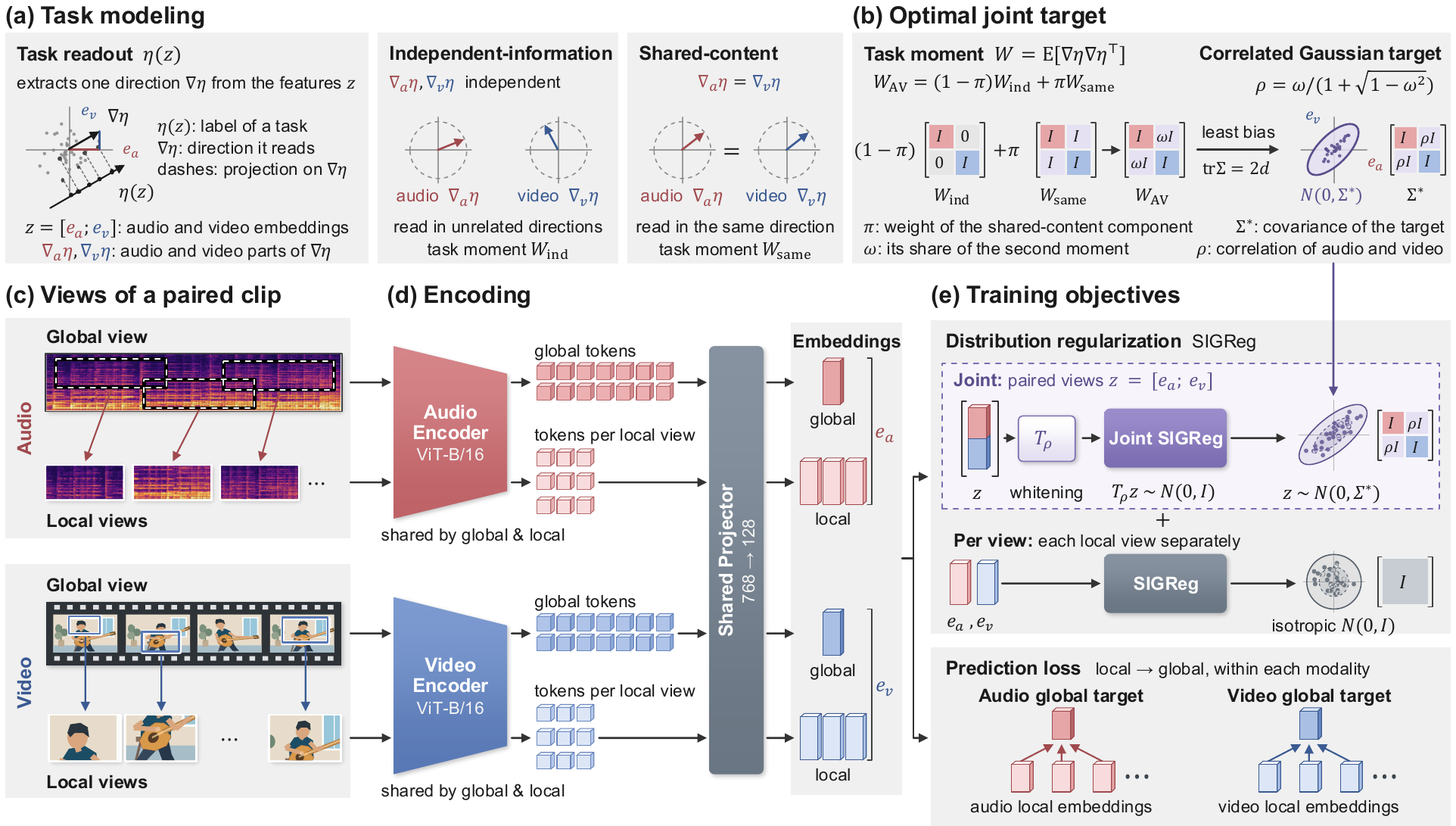}
 \vspace{-1.8em}
 \caption{Overview of MM-LeJEPA. (a) Task modeling: a task readout $\eta$ reads one direction $\nabla\eta$ of the features; independent-information tasks read audio and video in independent directions, shared-content tasks in the same direction. (b) The weighted task moment $W_{\mathrm{AV}}$ combines $W_{\mathrm{ind}}$ and $W_{\mathrm{same}}$ and yields the correlated Gaussian target $\mathcal{N}(0,\Sigma^{*})$ with correlation $\rho$. (c) Each paired clip gives a global view and local views per modality. (d) Modality encoders and a shared projector embed every view. (e) Training objectives: SIGReg on the whitened joint embedding $T_{\rho}\mathbf{z}$ and on each view separately, and a local-to-global prediction loss within each modality.}
 \label{fig:mmlejepa_pipeline}
\end{figure*}

\subsection{Audio-visual self-supervised pretraining}

Paired audio and video provide natural self-supervision, and existing objectives impose the relation between the modalities in three ways. The first is alignment, from correspondence and synchronization objectives~\citep{soundnet,looklistenlearn,objectsthatsound,avts,multisensory} to contrastive learning, clustering and refined positives~\citep{cmc,xdc,avid,robustavid,gdt}; CAV-MAE, MAViL and CAV-MAE Sync add it to masked reconstruction~\citep{cavmae,mavil,cavmaesync}, and EquiAV adds equivariance to augmentations~\citep{equiav}. Alignment pushes the two modalities toward full correlation rather than a correlation set by the task, and needs negatives or chosen positives.
The second is reconstruction, which leaves the relation implicit: Audiovisual MAE decodes both modalities~\citep{avmae} and MAViL adds a teacher-based self-training stage~\citep{mavil}. The third is architecture: MBT fuses through bottleneck tokens~\citep{mbt} and AVSiam shares one ViT~\citep{avsiam}, placing the relation in the network.
MM-LeJEPA instead states the relation as a property of the target distribution, a correlation $\rho$ derived from a task prior and enforced by a regularizer alone, with no negatives, decoder, teacher, or fusion module.
\subsection{JEPAs and LeJEPA}

JEPAs predict one view's embedding from another's instead of reconstructing inputs~\citep{lecun2022path}; their instances for images, video, audio, several modalities and audio-visual pairs~\citep{ijepa,vjepa,vjepa2,ajepa,data2vec,ibot,mjepa} all prevent collapse with stop-gradient and a moving-average teacher. A parallel line uses explicit constraints~\citep{simclr,moco,barlowtwins,vicreg,alignuniform,wmse}; Barlow Twins and VICReg are closest in form to our regularizer, acting on second moments{: Barlow Twins drives the cross-view correlation matrix to the identity, and VICReg pairs an invariance loss with variance and covariance terms on each branch. Both impose agreement between views as a heuristic. FactorCL~\citep{factorcl} separates task-relevant shared and unique information, while DeCUR~\citep{decur} targets identity and zero cross-correlation blocks for common and unique dimensions. Our task-gradient prior instead yields a Gaussian family with one tunable correlation $\rho$ between corresponding coordinates, regularized with SIGReg.}

LeJEPA derives an isotropic Gaussian target from a task-gradient prior and enforces it with SIGReg~\citep{lejepa}; {the prior is isotropic because nothing is assumed about the task}. LeVJEPA, AV-JEPA and {LeMuMoJEPA} carry this target to video, audio-visual and RGB--depth or thermal inputs~\citep{levjepa,avjepa,lemumojepa}. { AV-JEPA aligns global and modality-specific local views through a shared early-fusion encoder, while {LeMuMoJEPA} uses learnable fusion tokens; neither derives a joint target for embeddings from two separate encoders.}
MM-LeJEPA keeps one encoder per modality and changes only the prior{, so that a single scalar $\rho$ replaces} alignment losses, fusion layers, and heuristic decorrelation.
\section{Method}

\paragraph{Overview.} Fig.~\ref{fig:mmlejepa_pipeline} summarizes MM-LeJEPA from view construction to prediction and regularization. The framework applies to any modalities; in this paper we study audio and video.
We first revisit LeJEPA's prior over downstream tasks in Section~\ref{sec:method_prelim}, which yields the control quantity $J_W$ and the optimal distribution that constrains the embeddings. Section~\ref{sec:method_av_prior} formulates the multimodal task prior and states its assumptions for multimodal tasks. Section~\ref{sec:method_target} then minimizes $J_W$ under this prior to derive the optimal distribution of the multimodal embeddings, and Section~\ref{sec:method_sigreg} enforces this distribution on the embeddings through the training loss.
\subsection{LeJEPA's assumptions about downstream tasks}
\label{sec:method_prelim}
\signpost{Goal}{This subsection reviews LeJEPA's notation, assumptions and analysis to obtain the control quantity that links the embedding distribution to the task, the basis for the multimodal extension that follows.}

Let $\mathbf{z}=f_{\theta}(\mathbf{x})\in\mathbb{R}^{K}$ be the embedding of an input $\mathbf{x}$ and $p_z$ its distribution. As in LeJEPA, $p_z$ has zero mean and covariance $\Sigma\succ0$. The unknown labels of a downstream task are given by a target function $\eta:\mathbb{R}^{K}\to\mathbb{R}$, with gradient $\nabla\eta(\mathbf{q})$ at a query point $\mathbf{q}$. This gradient describes the task's local sensitivity along each embedding direction.

\paragraph{LeJEPA's objective.} We express LeJEPA's probe-bias analysis in a matrix-weighted form for task-averaged integrated squared bias (ISB). The analysis yields a quantity $J_W$ that depends on the embedding distribution and controls the ISB; minimizing it under the variance constraint gives the optimal distribution for this criterion.

We first consider a task prior fixed independently of the candidate representation distributions, and assume the task gradient has a zero first moment and a common positive-definite second moment:
\begin{equation}
 \mathbb{E}_{\eta}\!\left[\nabla\eta(\mathbf{q})\right]=0,\quad
 \mathbb{E}_{\eta}\!\left[\nabla\eta(\mathbf{q})\nabla\eta(\mathbf{q})^{\top}\right]=W\succ0.
 \label{eq:m_task_moments}
\end{equation}
Both conditions hold with the same $W$ for almost every $\mathbf{q}$, with finite second moments; a zero mean indicates no preferred signed gradient direction, and the form of $W$ is left unspecified.

For an embedding distribution $p_z$, let $\mathcal{I}(p_z):=\mathbb{E}_{\mathbf{z}\sim p_z}\!\left[\nabla\log p_z(\mathbf{z})\,\nabla\log p_z(\mathbf{z})^{\top}\right]$ be the Fisher information matrix. Following LeJEPA, we judge a representation by the task-averaged ISB of two nonparametric probes trained on the embedding, radius-based $k$-NN~\citep{cover1967nearest} and the Nadaraya--Watson kernel probe~\citep{nadaraya1964estimating,watson1964smooth}, whose bias depends only on the embedding distribution. With weights given by the task-gradient second moment $W$, define the weighted Fisher information functional (the second equality is proved in Appendix~\ref{app:proof_jw}):
\begin{align}
 J_W(p_z)
 &:=\int_{\mathbb{R}^{K}}
 \left(\nabla\log p_z(\mathbf{z})\right)^{\top}W\left(\nabla\log p_z(\mathbf{z})\right)p_z(\mathbf{z})\,\mathrm{d}\mathbf{z}\nonumber
 \\
 &=\operatorname{tr}\!\left(W\mathcal{I}(p_z)\right).
 \label{eq:m_jw}
\end{align}

 Under the task-gradient and curvature moment assumptions and each probe's regularity conditions, this functional controls the probes' task-averaged ISB over queries $\mathbf{q}\sim p_z$:
\begin{equation}
 \begin{aligned}
 \mathrm{ISB}_{k\mathrm{NN}}&=c_{k\mathrm{NN},0}+c_{k\mathrm{NN},1}\,J_W(p_z)+\varepsilon_{k\mathrm{NN}},
 \\
 \mathrm{ISB}_{\mathrm{ker}}&\le c_{\mathrm{ker},0}+c_{\mathrm{ker},1}\,J_W(p_z)+\varepsilon_{\mathrm{ker}}.
 \end{aligned}
\end{equation}
Both relations are derived in Appendix~\ref{app:proof_probes}. Each $c$ is a nonnegative constant independent of the candidate distribution's shape, and each $\varepsilon$ is an infinitesimal term; the distribution-dependent design term is $J_W(p_z)$.

\paragraph{Optimizing the control quantity $J_W$.} {Following LeJEPA's distribution-design strategy, we minimize} $J_W(p_z)$, which controls the task-averaged ISB, to obtain the optimal shape of the embedding distribution.

LeJEPA further assumes an isotropic second moment to model unknown downstream tasks without favoring any gradient direction. That is, $W$ in Eq.~(\ref{eq:m_task_moments}) is taken to be $W=\tau_g^{2}I_K$ with $0<\tau_g^{2}<\infty$.

Under this isotropic prior, Eq.~(\ref{eq:m_jw}) reduces to
\begin{equation}
 \begin{aligned}
 J_W(p_z)&=\tau_g^{2}\operatorname{tr}\mathcal{I}(p_z)=\tau_g^{2}J(p_z),\\
 J(p_z)&:=\int_{\mathbb{R}^{K}}\|\nabla\log p_z(\mathbf{z})\|^{2}p_z(\mathbf{z})\,\mathrm{d}\mathbf{z},
 \end{aligned}
\end{equation}
where $J(p_z)$ is the Fisher information functional of LeJEPA. LeJEPA optimizes over all distributions of $\mathbf{z}$ under a fixed total variance, which rules out shrinking $J$ by inflating the scale. Taking it as $\operatorname{tr}(\Sigma)=K$, \textit{i.e.}\ unit average coordinate variance,
\begin{equation}
 J_W(p_z)=\tau_g^{2}J(p_z)\;\ge\;\tau_g^{2}\operatorname{tr}\!\left(\Sigma^{-1}\right)\;\ge\;\tau_g^{2}K.
\end{equation}

The two inequalities hold with equality if and only if $p_z=\mathcal{N}(0,\Sigma)$ and $\Sigma=I_K$, respectively; proof references for both inequalities and their equality conditions are given in Appendix~\ref{app:proof_isotropic}. Therefore, when $J_W(p_z)$ attains its minimum, $p_z$ must satisfy both conditions, \textit{i.e.}, $p_z^{*}=\mathcal{N}(0,I_K)$.

LeJEPA then constrains the embeddings with SIGReg~\citep{lejepa}, which drives the distribution to the standard Gaussian by matching random one-dimensional projections with a univariate normality test~\citep{epps1983}, so that the embeddings take the optimal distribution shape.

\signpost{Takeaway}{Under LeJEPA's assumptions, the task-averaged ISB is controlled by $J_W$, a quantity that depends on $p_z$ and on the task prior $W$, so minimizing the ISB reduces to minimizing $J_W$. In the unimodal case LeJEPA chooses an isotropic $W$, which gives the isotropic Gaussian target $p_z^{*}$; in the multimodal case a task prior that couples the two modalities changes $W$ and with it the target, which the next subsections derive.}

\subsection{{Multimodal} task prior and the second moment of the joint task gradient}
\label{sec:method_av_prior}
\signpost{Goal}{{Multimodal tasks differ from unimodal ones in that the modalities jointly describe part of the same content. This subsection modifies the task assumptions accordingly and obtains the task representation $W_{\mathrm{AV}}$.}}

\paragraph{From unimodal to multimodal tasks.} {Section~\ref{sec:method_prelim} shows that the target is fixed by the second moment $W$ of the task gradient, and LeJEPA takes $W$ isotropic because nothing is known about the task. With two modalities this choice is no longer neutral. The coordinates of the joint embedding come from two encoders, and audio and video often observe the same event, so a task that reads such shared content has gradients on the two sides that vary together; its second moment has cross-modal blocks that an isotropic $W$ sets to zero. An isotropic $W$ therefore no longer applies to multimodal tasks and must be modified accordingly. We keep LeJEPA's route, a task prior that fixes $W$ and the induced functional $J_W$, and change only the prior so that $W$ carries the cross-modal structure of shared-content tasks.}

\paragraph{Multimodal task prior.} For an audio-visual task, let the audio and video embeddings, the outputs of the projector on which the regularizer acts (Section~\ref{sec:method_sigreg}), be $\mathbf{e}_a,\mathbf{e}_v\in\mathbb{R}^{d}$ and their joint embedding be $\mathbf{z}=[\mathbf{e}_a;\mathbf{e}_v]\in\mathbb{R}^{2d}$. At a fixed query point $\mathbf{q}$, we split the task gradient into blocks with respect to $\mathbf{e}_a$ and $\mathbf{e}_v$:
\begin{equation}
  \nabla\eta(\mathbf{q})=\left[\nabla_a\eta^{\top},\ \nabla_v\eta^{\top}\right]^{\top},
  \quad \nabla_a\eta,\nabla_v\eta\in\mathbb{R}^{d}.
\end{equation}

Audio and video give different observations of the same content. We model a downstream task through independent-information and shared-content components, combining their gradients with different strengths:
\begin{equation}
 \nabla\eta=\sqrt{1-\pi}\,\nabla\eta_{\mathrm{ind}}
 +\sqrt{\pi}\,\nabla\eta_{\mathrm{same}},
 \quad 0\le\pi<1 .
\end{equation}

Here $\eta_{\mathrm{ind}}$ reads independent information from the two modalities, and $\eta_{\mathrm{same}}$ reads the same content from both. Their combination defines the task prior, with $\pi$ the weight of the shared-content component across all downstream tasks. The condition $\pi<1$ excludes the degenerate prior in which every task reads shared content only.
For example, identifying a guitar from its sound and image reads shared content, whereas recognizing background speech from audio and the shirt color from video reads independent information.

\paragraph{Fixed correspondence of multimodal content.} The same content can be read out from both audio and video, but the directions and scales in which the two encoders represent it are arbitrary. We proceed in two steps. First, we assume that the correspondence between the two representations of shared content is fixed: it is set by how the two encoders represent the content, and it does not depend on the task. Second, because a shared-content task reads the same content on both sides, its audio and video gradients are related by this same fixed correspondence, which is the form in which the assumption is used below. The {diagonal-times-orthogonal} form below loses no generality: any invertible linear correspondence can be brought to it by a singular value decomposition together with a rotation of the video coordinates, which their arbitrariness permits. A single change of coordinates then aligns the two sides so that the gradient of shared content is the same in both modalities, and the task prior is formulated in terms of this fixed relation.

\begin{assumption}[fixed correspondence of shared content]
\label{as:correspondence}
There exist a fixed orthogonal $R\in O(d)$ and a fixed positive diagonal $\Lambda\succ0$, independent of the task and the query point, such that
\begin{equation}
 \nabla_v\eta_{\mathrm{same}}=\Lambda R\,\nabla_a\eta_{\mathrm{same}}.
\end{equation}
\end{assumption}

\begingroup
An invertible change of coordinates does not change what the content can express, so it leaves every task readout unchanged, and maps each candidate embedding distribution to exactly one distribution in the new coordinates. We therefore pose the distribution-design problem in the common coordinates defined below, without affecting what the distribution can express.
In Assumption~\ref{as:correspondence}, $R$ describes the directional correspondence of shared-content readouts, while $\Lambda$ accounts for their coordinate-wise scaling. We absorb these fixed differences into aligned video coordinates $\widetilde{\mathbf{e}}_v=R^{\top}\Lambda\mathbf{e}_v$. The same shared-content readout, expressed in $(\mathbf{e}_a,\widetilde{\mathbf{e}}_v)$, then satisfies
\begin{equation}
 \nabla_{\widetilde{\mathbf{e}}_v}\eta_{\mathrm{same}}
 =\nabla_a\eta_{\mathrm{same}}.
 \label{eq:m_aligned}
\end{equation}
All subsequent moments, probe calculations and variance constraints are defined in these common coordinates, and we henceforth write $\mathbf{e}_v$ for the aligned video coordinates; no per-task gradient correspondence is required for the independent-information component.
\par
\endgroup

\paragraph{General assumptions on multimodal tasks.} Relative to LeJEPA's prior over arbitrary tasks, the multimodal prior adds one distinction only, between shared-content and independent-information tasks. Everything else is kept: the downstream task is still unknown in advance, so no preference is placed on the strength or direction of the task gradient or on either modality. The following assumption states this for the two components in the common coordinates.
\begin{assumption}[moments and independence of task gradients]
\label{as:generality}

For $m,m'\in\{a,v\}$, the independent-information component has independent modality gradients with
\begin{equation}
 \begin{aligned}
 \mathbb{E}_{\eta}[\nabla_m\eta_{\mathrm{ind}}]&=0,\\
 \mathbb{E}_{\eta}[\nabla_m\eta_{\mathrm{ind}}\,\nabla_{m'}\eta_{\mathrm{ind}}^{\top}]&=\delta_{mm'}\,\tau_{\mathrm{ind}}^{2}I_d,
 \end{aligned}
\end{equation}
where $\delta_{mm'}=1$ if $m=m'$ and $\delta_{mm'}=0$ otherwise. The shared-content component has gradients that coincide in the common coordinates by Eq.~(\ref{eq:m_aligned}), with
\begin{equation}
 \begin{aligned}
 \mathbb{E}_{\eta}[\nabla_m\eta_{\mathrm{same}}]&=0,\\
 \mathbb{E}_{\eta}[\nabla_m\eta_{\mathrm{same}}\,\nabla_{m'}\eta_{\mathrm{same}}^{\top}]&=\tau_{\mathrm{same}}^{2}I_d.
 \end{aligned}
\end{equation}
The two components are independent, hence $\mathbb{E}_{\eta}[\nabla\eta_{\mathrm{ind}}\nabla\eta_{\mathrm{same}}^{\top}]=0$. Here $0<\tau_{\mathrm{ind}}^{2},\tau_{\mathrm{same}}^{2}<\infty$.
\end{assumption}

Under the assumptions above, let $W_{\mathrm{AV}}:=\mathbb{E}_{\eta}[\nabla\eta\nabla\eta^{\top}]$ be the joint task-gradient second moment.

\begin{proposition}
\label{prop:constant_spectrum}
Under Assumptions~\ref{as:correspondence} and~\ref{as:generality}, the component gradient second moments are
\begin{equation}
 W_{\mathrm{ind}}=\tau_{\mathrm{ind}}^{2}
 \begin{bmatrix}I_d&0\\ 0&I_d\end{bmatrix},
 \quad
 W_{\mathrm{same}}=\tau_{\mathrm{same}}^{2}
 \begin{bmatrix}I_d&I_d\\ I_d&I_d\end{bmatrix},
\end{equation}
so that the second moment of the joint task gradient is
\begin{equation}
 W_{\mathrm{AV}}=(1-\pi)W_{\mathrm{ind}}+\pi W_{\mathrm{same}}
 =\tau_{\mathrm{AV}}^{2}\begin{bmatrix}I_d&\omega I_d\\ \omega I_d&I_d\end{bmatrix},
 \label{eq:m_wav}
\end{equation}
with $\tau_{\mathrm{AV}}^{2}=(1-\pi)\tau_{\mathrm{ind}}^{2}+\pi\tau_{\mathrm{same}}^{2}>0$ and $\omega=\pi\tau_{\mathrm{same}}^{2}/\tau_{\mathrm{AV}}^{2}\in[0,1)$.
\end{proposition}

The second-moment expression is derived in Appendix~\ref{app:proof_moment}.
At $\omega=0$, $W_{\mathrm{AV}}=\tau_{\mathrm{AV}}^{2}I_{2d}$ recovers LeJEPA's isotropic second-moment assumption over the full embedding space in Section~\ref{sec:method_prelim}. This isotropic prior ignores cross-modal coupling, whereas $\omega$ controls the contribution of shared-content components to the gradient second moment.

\signpost{Takeaway}{{This section derived the task prior for multimodal tasks and its second moment $W_{\mathrm{AV}}$. As in LeJEPA, the next section minimizes the Fisher information functional $J_{W_{\mathrm{AV}}}$ induced by this prior under the total-variance constraint to obtain the optimal embedding distribution.}}

\subsection{The optimal joint target distribution}
\label{sec:method_target}
\signpost{Goal}{With $W_{\mathrm{AV}}$ in hand, this section solves for the optimal distribution under the multimodal prior: among all embedding distributions with a fixed total variance, find the one that minimizes $J_{W_{\mathrm{AV}}}$, the quantity that controls the probes' bias.}

\paragraph{{Optimizing} the control quantity $J_{W_{\mathrm{AV}}}$.} With $W_{\mathrm{AV}}$ from Proposition~\ref{prop:constant_spectrum}, the weighted Fisher information functional of Eq.~(\ref{eq:m_jw}) reads $J_{W_{\mathrm{AV}}}(p_z)=\operatorname{tr}(W_{\mathrm{AV}}\mathcal{I}(p_z))$; we minimize it under a fixed total-variance constraint to obtain the optimal shape of the embedding distribution.

\begin{proposition}
\label{prop:weighted_cr}
Let $p_z$ be an everywhere-positive density on $\mathbb{R}^{2d}$ with zero mean and covariance $\Sigma\succ0$. Assume $p_z\in C^{3}$ with bounded derivatives up to order three, and assume finite Fisher information and covariance. Then
\begin{equation}
 J_{W_{\mathrm{AV}}}(p_z)\ \ge\ \operatorname{tr}\!\left(W_{\mathrm{AV}}\Sigma^{-1}\right),
\end{equation}
and, when $W_{\mathrm{AV}}\succ0$, equality holds if and only if $p_z=\mathcal{N}(0,\Sigma)$. The proof is given in Appendix~\ref{app:proof_cr}.
\end{proposition}

As in LeJEPA, we impose the total-variance constraint $\operatorname{tr}(\Sigma)=2d$, which only fixes the overall scale; the optimal covariance then takes the form below, with unit variance per coordinate.

\begin{proposition}
\label{prop:target}
Let $W_{\mathrm{AV}}$ have the form in Eq.~(\ref{eq:m_wav}), with $\omega\in[0,1)$.
Under the total-variance constraint above, the unique minimizer of $\operatorname{tr}(W_{\mathrm{AV}}\Sigma^{-1})$ over $\Sigma\succ0$ is
\begin{equation}
 \Sigma^{*}(\rho)=\begin{bmatrix}I_d&\rho I_d\\ \rho I_d&I_d\end{bmatrix},
 \quad
 \rho=\frac{\omega}{1+\sqrt{1-\omega^{2}}}.
 \label{eq:m_target}
\end{equation}

\end{proposition}
The covariance and correlation formulas are derived in Appendix~\ref{app:proof_target}.

Propositions~\ref{prop:weighted_cr} and~\ref{prop:target} show that, under the total-variance constraint, $\mathcal{N}(0,\Sigma^{*}(\rho))$ uniquely minimizes $J_{W_{\mathrm{AV}}}(p_z)$; we adopt it as our target embedding distribution. The task-prior parameters $\pi$, $\tau_{\mathrm{ind}}^{2}$ and $\tau_{\mathrm{same}}^{2}$ are unknown before pretraining, so we retain this correlated Gaussian family and treat $\rho$ as an empirical hyperparameter. The derivation fixes the shape of the target, isotropic within each modality and one common correlation for every coordinate pair, so the $2d\times 2d$ covariance is reduced to this single scalar.
\signpost{Takeaway}{Under the task prior of the previous subsection, the optimal target is a correlated Gaussian whose only free parameter is the cross-modal correlation $\rho$.}

\subsection{Loss function and SIGReg regularizer}
\label{sec:method_sigreg}
\signpost{Goal}{{This subsection uses SIGReg and the optimal shape of Eq.~(\ref{eq:m_target}) to constrain the embedding distribution, and gives the final objective.}}

SIGReg regularizes embeddings toward a standard Gaussian~\citep{lejepa}. To adapt it to our correlated target, we first apply a fixed linear whitening transform~\citep{kessy2018whitening} that maps the target covariance to the identity. We then apply SIGReg to the transformed embeddings. For the joint projection $\mathbf{z}=[\mathbf{e}_a;\mathbf{e}_v]$ and $|\rho|<1$, the transform is
\begin{equation}
 T_{\rho}\mathbf{z}
 =\left[\frac{\mathbf{e}_a+\mathbf{e}_v}{\sqrt{2(1+\rho)}}\ ;\ \frac{\mathbf{e}_a-\mathbf{e}_v}{\sqrt{2(1-\rho)}}\right].
\end{equation}
Since $T_{\rho}$ is invertible, $T_{\rho}\mathbf{z}\sim\mathcal{N}(0,I_{2d})$ if and only if $\mathbf{z}\sim\mathcal{N}(0,\Sigma^{*}(\rho))$, so constraining the transformed embedding to a standard Gaussian enforces Eq.~(\ref{eq:m_target}). The target was derived in the aligned coordinates of Section~\ref{sec:method_av_prior}; since representations related by an invertible linear map are equivalent for every task readout, the projector can realize that alignment itself, so we enforce the target in the projector's coordinates and let training absorb $R$ and $\Lambda$.

As in LeJEPA, besides the shape constraint we use a local-to-global prediction loss, in which each local view predicts the embedding of its global view, similar to the local-to-global correspondence in multi-crop training~\citep{swav,dino}; this keeps the content of each view in the embedding.
Concretely, for each audio-visual pair we construct local views of the audio and of the video, and for every local view a corresponding global view of the same modality; each local view predicts the embedding of its global view, and several local views may share one. With $V_l$ local views and $\mathbf{g}_{i}$ the global embedding of local view $i$, the prediction term is the coordinate-wise mean squared error:
\begin{equation}
 \mathcal{L}_{\mathrm{pred}}
 =\frac{1}{dV_l}\sum_{i=1}^{V_l}\left\|\mathbf{z}_i-\mathbf{g}_{i}\right\|^{2}.
\end{equation}
The distributional term $\mathcal{L}_{\mathrm{SIGReg}}$ applies SIGReg with the target derived above to the global and to the local embeddings. The overall objective is
\begin{equation}
 \mathcal{L}=(1-\lambda)\mathcal{L}_{\mathrm{pred}}+\lambda\mathcal{L}_{\mathrm{SIGReg}},
 \quad \lambda\in[0,1].
\end{equation}

\signpost{Takeaway}{{As in LeJEPA, the objective combines a prediction term, which keeps the content of each view in the embedding, with the SIGReg regularizer, which shapes the joint distribution toward the correlated Gaussian target. Relative to LeJEPA, the only added parameter is the target correlation $\rho$.}}
\begin{table*}[t]
 \centering
    \vspace{-1em}
 \small
 \caption{{Frozen-backbone probing: VGGSound top-1 accuracy and AudioSet-20K mAP (\%, $\uparrow$). Attentive probing trains an attention-pooling head on the frozen encoder; linear probing trains a linear layer on mean-pooled features. Bold: best, underline: second best.}}
 \vspace{-1em}
 \label{tab:probe}
 \setlength{\tabcolsep}{1.5pt}
 \begin{tabular*}{0.82\textwidth}{@{\extracolsep{\fill}}ll ccc ccc ccc ccc@{}}
 \toprule
 & & \multicolumn{6}{c}{Attentive probing} & \multicolumn{6}{c}{Linear probing} \\
 \cmidrule(lr){3-8}\cmidrule(lr){9-14}
 & & \multicolumn{3}{c}{VGGSound} & \multicolumn{3}{c}{AudioSet-20K} & \multicolumn{3}{c}{VGGSound} & \multicolumn{3}{c}{AudioSet-20K} \\
 \cmidrule(lr){3-5}\cmidrule(lr){6-8}\cmidrule(lr){9-11}\cmidrule(lr){12-14}
 Method & Backbone & A & V & A-V & A & V & A-V & A & V & A-V & A & V & A-V \\
 \midrule
 CAV-MAE~\citep{cavmae} & ViT-B & 40.7 & 30.2 & 46.1 & 33.3 & 18.8 & \underline{36.6} & 39.2 & 30.8 & 43.8 & \underline{32.7} & 18.1 & \underline{35.1} \\
 EquiAV~\citep{equiav} & ViT-B & \underline{41.2} & 31.6 & \underline{46.3} & \underline{33.5} & 19.9 & \underline{36.6} & \underline{41.3} & 31.7 & \underline{46.0} & 32.2 & 18.7 & 34.9 \\
 AVSiam~\citep{avsiam} & ViT-B & 36.9 & \underline{32.7} & 43.8 & 29.6 & \underline{21.7} & 34.0 & 37.6 & \underline{34.1} & 42.9 & 30.7 & \textbf{22.6} & 34.1 \\
 CAV-MAE Sync~\citep{cavmaesync} & ViT-B & 35.0 & 23.0 & 40.3 & 26.6 & 12.2 & 29.4 & 35.2 & 23.7 & 40.3 & 26.1 & 11.9 & 28.1 \\
 \midrule
 MM-LeJEPA & ViT-B & \textbf{44.7} & \textbf{35.7} & \textbf{50.3} & \textbf{37.2} & \textbf{22.7} & \textbf{40.2} & \textbf{43.3} & \textbf{36.5} & \textbf{49.3} & \textbf{35.3} & \underline{22.2} & \textbf{38.4} \\
 \bottomrule
 \end{tabular*}
\end{table*}

\begin{table*}[t]
 \centering
 \vspace{0em}
 \small
 \caption{Fine-tuning on VGGSound (top-1 accuracy) and AudioSet (mAP) (\%, $\uparrow$). $\ast$: A and V from single-modality inference of the audio-visual model. $\ddagger$: supervised ImageNet-21K initialization. Bold: best, underline: second best.}
 \vspace{-1em}
 \label{tab:ft}
 \setlength{\tabcolsep}{2pt}
 \begin{tabular*}{0.82\textwidth}{@{\extracolsep{\fill}}llccccccccc@{}}
 \toprule
 & & \multicolumn{3}{c}{VGGSound} & \multicolumn{3}{c}{AudioSet-20K} & \multicolumn{3}{c}{AudioSet-2M} \\
 \cmidrule(lr){3-5}\cmidrule(lr){6-8}\cmidrule(lr){9-11}
 Method & Backbone & A & V & A-V & A & V & A-V & A & V & A-V \\
 \midrule

 AV-MAE~\citep{avmae} & ViT-B & --- & --- & 64.2 & --- & --- & --- & --- & --- & 50.0 \\
 CAV-MAE~\citep{cavmae} & ViT-B & 59.5 & 47.0 & 65.5 & 37.7 & 19.8 & 42.0 & 46.6 & 26.2 & 51.2 \\
 MAViL~\citep{mavil} & ViT-B & 60.8 & \underline{50.9} & \underline{67.1} & 41.8 & 24.8 & 44.9 & \underline{48.7} & \underline{30.3} & 53.3 \\
 EquiAV~\citep{equiav} & ViT-B & \underline{61.0} & 50.7 & \underline{67.1} & \underline{42.4} & \textbf{25.7} & \underline{46.6} & \textbf{49.1} & 30.1 & \textbf{54.6} \\
 AVSiam$^{\ast\ddagger}$~\citep{avsiam} & ViT-B & 55.7 & 46.0 & 64.9 & 36.5 & 23.7 & 41.6 & 45.2 & \textbf{30.8} & 50.1 \\
 AV-JEPA$^{\ast}$~\citep{avjepa} & ViT-B & --- & --- & 57.1 & 23.7 & 10.3 & 29.6 & 26.0 & 12.8 & 32.7 \\
 \midrule
 MM-LeJEPA & ViT-B & \textbf{61.2} & \textbf{51.0} & \textbf{67.4} & \textbf{42.6} & \underline{25.0} & \textbf{46.7} & 48.5 & 30.2 & \underline{53.4} \\
 \bottomrule
 \end{tabular*}
\end{table*}
\section{Experiments}

\subsection{Experimental setup}

{We pretrain on AudioSet~\citep{audioset} without labels; because some videos are no longer available, our copies are smaller than the official lists. We evaluate on VGGSound (VGG)~\citep{vggsound}, 309 single-label classes scored by top-1 accuracy, and on AudioSet, fine-tuning on its balanced AudioSet-20K (AS-20K) or full AudioSet-2M training set and scoring mean average precision (mAP) over 527 labels. A, V and A-V denote audio-only, video-only and audio-visual inputs}.

{Each modality uses a separate 12-block ViT-B/16 encoder~\citep{vit,cavmae} initialized from ImageNet MAE weights~\citep{mae}, and a shared projector maps mean-pooled 768-dimensional features to 128-dimensional embeddings. Audio is represented as a log-mel spectrogram with 1{,}024 time steps and 128 mel bins, and video frames are resized to $224\times224$. Each clip gives one global audio view and nine local audio windows, plus four global and nine local video frames drawn from a pool of the clip's frames; each local view predicts a global embedding of its own modality. We pretrain for 30 epochs with AdamW~\citep{adamw} at learning rate $2\times10^{-4}$ and batch size 256, with $\lambda=0.05$ and $\rho=0.8$}.

\subsection{Frozen-backbone probing}
{The best way to test the shape of the learned embeddings is frozen probing. It evaluates the pretrained representation without updating the encoder, and it is the standard protocol of I-JEPA, V-JEPA, V-JEPA~2 and LeJEPA~\citep{ijepa,vjepa,vjepa2,lejepa}. {For every method, attentive probing trains an attention-pooling head~\citep{cae,vjepa} on final-layer patch tokens for 20 epochs}, and linear probing trains a single layer on mean-pooled features; all probing baselines use their released weights and our code. MM-LeJEPA tops all six attentive-probing columns of Table~\ref{tab:probe}; on audio-visual inputs it reaches 50.3 on VGGSound and 40.2 on AudioSet-20K, while the best baselines reach 46.3 and 36.6. Linear probing tells the same story with one exception, video-only AudioSet-20K, where AVSiam, which shares one transformer across modalities, scores 22.6 to our 22.2. The advantage is already there before any fine-tuning}.
\begin{table*}[t]
 \centering
\vspace{0em}
 \caption{VGGSounder results on the background-music-free subset used by the benchmark (left) and on the full test set (right) (v1.0 labels, top-1, \%, $\uparrow$). A, V and A-V denote audio, visual and audio-visual inputs. Bold: best, underline: second best.}
 \vspace{-1em}
 \label{tab:vggsounder}
 \small
 \setlength{\tabcolsep}{1.5pt}
 \begin{tabular*}{\textwidth}{@{\extracolsep{\fill}}l ccc ccc ccc ccc ccc ccc@{}}
 \toprule
 & \multicolumn{9}{c}{Background-music-free subset} & \multicolumn{9}{c}{Full test set} \\
 \cmidrule(lr){2-10}\cmidrule(lr){11-19}
 & \multicolumn{3}{c}{Subset acc.} & \multicolumn{3}{c}{{micro-F1}} & \multicolumn{3}{c}{Hit rate} & \multicolumn{3}{c}{Subset acc.} & \multicolumn{3}{c}{{micro-F1}} & \multicolumn{3}{c}{Hit rate} \\
 \cmidrule(lr){2-4}\cmidrule(lr){5-7}\cmidrule(lr){8-10}\cmidrule(lr){11-13}\cmidrule(lr){14-16}\cmidrule(lr){17-19}
 Method & A & V & A-V & A & V & A-V & A & V & A-V & A & V & A-V & A & V & A-V & A & V & A-V \\
 \midrule
 CAV-MAE~\citep{cavmae} & \underline{17.9} & 24.2 & \underline{29.6} & \underline{39.5} & 38.9 & \underline{47.2} & \underline{67.1} & 56.4 & \underline{67.5} & \underline{16.7} & 24.2 & \underline{29.4} & \underline{38.1} & 38.5 & \underline{46.7} & \underline{65.1} & 55.2 & \underline{65.9} \\
 EquiAV~\citep{equiav} & 15.8 & 13.6 & 24.4 & 33.7 & 22.9 & 39.3 & 57.1 & 33.3 & 56.2 & 14.6 & 13.4 & 24.0 & 32.1 & 22.4 & 38.5 & 54.8 & 32.2 & 54.4 \\
 AVSiam~\citep{avsiam} & 17.4 & \underline{25.0} & 27.9 & 38.1 & \underline{39.8} & 43.7 & 64.7 & \underline{57.7} & 62.5 & 16.1 & \underline{24.7} & 27.2 & 36.6 & \underline{39.0} & 42.7 & 62.6 & \underline{55.9} & 60.4 \\
 DeepAVFusion~\citep{mo2024deepavfusion} & 13.9 & 13.8 & 26.2 & 29.1 & 23.3 & 41.8 & 49.4 & 33.8 & 59.8 & 12.8 & 13.6 & 25.8 & 27.5 & 22.7 & 40.9 & 47.1 & 32.5 & 57.8 \\
 \midrule
 MM-LeJEPA & \textbf{18.3} & \textbf{26.6} & \textbf{30.2} & \textbf{40.0} & \textbf{42.3} & \textbf{48.4} & \textbf{67.9} & \textbf{61.3} & \textbf{69.1} & \textbf{17.0} & \textbf{26.6} & \textbf{30.0} & \textbf{38.6} & \textbf{41.7} & \textbf{47.8} & \textbf{65.9} & \textbf{59.8} & \textbf{67.5} \\
 \bottomrule
 \end{tabular*}
\end{table*}

\begin{table*}[t]
 \centering\small

 \caption{Sensitivity to the regularization weight $\lambda$ (left), batch size (middle) and target correlation $\rho$ (right). Bold: best, underline: second best.}
 \vspace{-1em}
 \label{tab:sens}
 \setlength{\tabcolsep}{3pt}
 \begin{tabular*}{\textwidth}{@{\extracolsep{\fill}}c cc cc @{\hspace{20pt}} c cc cc @{\hspace{20pt}} c cc cc@{}}
 \toprule
 & \multicolumn{2}{c}{\makebox[0pt]{Attentive probing}} & \multicolumn{2}{c}{Fine-tuning} & & \multicolumn{2}{c}{\makebox[0pt]{Attentive probing}} & \multicolumn{2}{c}{Fine-tuning} & & \multicolumn{2}{c}{\makebox[0pt]{Attentive probing}} & \multicolumn{2}{c}{Fine-tuning} \\
 \cmidrule(lr){2-3}\cmidrule(lr){4-5}\cmidrule(lr){7-8}\cmidrule(lr){9-10}\cmidrule(lr){12-13}\cmidrule(lr){14-15}
 $\lambda$ & VGG & AS-20K & VGG & AS-20K & Batch & VGG & AS-20K & VGG & AS-20K & $\rho$ & VGG & AS-20K & VGG & AS-20K \\
 \midrule
 0.025 & 50.0 & 39.8 & 66.9 & 45.3 & 64 & 49.7 & 39.5 & 66.1 & 44.7 & 0.00 & 42.9 & 31.6 & 65.8 & 38.7 \\
 0.050 & \textbf{50.3} & \textbf{40.2} & \textbf{67.4} & \textbf{46.7} & 128 & \underline{50.2} & 39.9 & 67.1 & 46.6 & 0.40 & \underline{49.7} & \underline{39.1} & \underline{66.9} & \underline{46.5} \\
 0.100 & \underline{50.1} & \underline{40.0} & \underline{67.0} & \underline{46.2} & 256 & \textbf{50.3} & \textbf{40.2} & \textbf{67.4} & \underline{46.7} & 0.80 & \textbf{50.3} & \textbf{40.2} & \textbf{67.4} & \textbf{46.7} \\
 0.200 & 49.2 & 38.3 & 66.6 & 45.1 & 512 & \textbf{50.3} & \underline{40.1} & \underline{67.3} & \textbf{46.8} & 0.99 & 48.7 & 37.6 & 65.9 & 44.8 \\
 \bottomrule
 \end{tabular*}
\end{table*}

\begin{table}[t]
 \centering\small\setlength{\tabcolsep}{3pt}
 \caption{Component ablations. Bold: best, underline: second best.}
 \vspace{-1em}
 \label{tab:ablation}
 \begin{tabular}{lcccc}
 \toprule
 & \multicolumn{2}{c}{\makebox[0pt]{Attentive probing}} & \multicolumn{2}{c}{Fine-tuning} \\
 \cmidrule(lr){2-3}\cmidrule(lr){4-5}
 Setting & VGG & AS-20K & VGG & AS-20K \\
 \midrule
 Full objective & \textbf{50.3} & \textbf{40.2} & \textbf{67.4} & \textbf{46.7} \\
 Isotropic joint target ($\rho=0$) & 42.9 & 31.6 & 65.8 & 38.7 \\
 Without joint term & 43.4 & 33.6 & 66.2 & 39.9 \\
 Global only & \underline{49.8} & \underline{39.4} & \underline{66.9} & \underline{44.2} \\
 Local only & 48.5 & 38.8 & 66.5 & 43.7 \\
 Without prediction & 12.4 & 6.5 & 43.6 & 6.9 \\
 Without SIGReg & 16.9 & 3.2 & 41.7 & 3.4 \\
 \bottomrule
 \end{tabular}
 \vspace{-1em}
\end{table}

\subsection{Fine-tuning on three datasets}
{Fine-tuning is a harder test, since training the whole network can wash out differences in the starting features. We attach a classification head and train with AdamW at $10^{-4}$ for 10 epochs on VGGSound and Adam~\citep{adam} at $1.5\times10^{-4}$ for 20 epochs on AudioSet-20K and at $10^{-5}$ for 10 epochs on AudioSet-2M, with a cosine schedule, or with per-epoch halving after the second epoch on AudioSet-2M. Baselines in Table~\ref{tab:ft} are quoted from their papers. The margins narrow, yet MM-LeJEPA still gives the best audio-visual results, 67.4 on VGGSound and 46.7 on AudioSet-20K, without negatives, an EMA teacher or self-training. On AudioSet-2M, EquiAV stays ahead; our copy of that training set has 1{,}785{,}610 clips against EquiAV's 1{,}893{,}278~\citep{equiav}, which may explain part of the gap}.

{VGGSounder~\citep{zverevwiedemer2025vggsounder} gives a finer view. It relabels the VGGSound test set with several labels per clip, each marked as audible, visible or both, so that each input is judged only on what it can actually reveal. Following its protocol, we report subset accuracy, micro-F1 and hit rate of the top-1 predictions of our VGGSound models with the v1.0 labels for each input, on the benchmark's background-music-free subset and on the full test set; subset baselines come from the benchmark paper, and full-set baselines are recomputed from the released per-clip predictions with the official scorer. MM-LeJEPA leads every column of Table~\ref{tab:vggsounder}}.

\subsection{{Sensitivity to target correlation, regularization weight {and} batch size}}
{$\rho$ is the only knob MM-LeJEPA adds to LeJEPA, and the right panel of Table~\ref{tab:sens} varies it first. Setting $\rho=0$ gives back LeJEPA's isotropic target, and it hurts: attentive probing falls from 50.3 to 42.9 on VGGSound and from 40.2 to 31.6 on AudioSet-20K. Pushing $\rho$ to 0.99 overshoots and lowers every score as well, so among the values we tried the best target sits in between, at $\rho=0.8$. The embeddings show what changes: at $\rho=0.8$ audio and video share neighbourhoods, and among each embedding's ten nearest neighbours in the 128-dimensional space, the share from the other modality rises from 24\% at $\rho=0$ to 46\%, close to the 50\% of a perfect mix. The other two settings are far less delicate. Halving or quadrupling $\lambda$ lowers all four scores, while $\lambda=0.1$ stays close to the reference, and batch sizes from 128 to 512 give nearly the same scores, with only a batch of 64 falling behind}.

\subsection{Component ablations}
{Table~\ref{tab:ablation} removes the parts of the objective one at a time. In the full objective, SIGReg has a joint term on the whitened audio-visual pairs and a marginal term on every view. Dropping the joint term and keeping only the marginal ones brings attentive probing down to 43.4 on VGGSound and 33.6 on AudioSet-20K, close to the 42.9 and 31.6 of the isotropic target, so shaping each modality on its own is not enough; the gain comes mainly from constraining the cross-modal correlation. Where the regularizer acts matters less: restricting it to the global or to the local embeddings lowers every score far less, AudioSet-20K fine-tuning the most, to 44.2 and 43.7. The prediction loss and SIGReg are both indispensable: without either one, every score falls far below the rest of the table}.
\section{Conclusion}
{Applying LeJEPA's isotropic target to a joint multimodal embedding suppresses the cross-modal correlation that shared content creates, and we removed this mismatch by changing the task prior. We modeled downstream tasks as mixtures of independent-information and shared-content readouts, derived the correlated Gaussian target that minimizes the weighted Fisher information under a fixed total variance, and enforced it with SIGReg through a fixed whitening transform. MM-LeJEPA therefore adds one scalar, the target correlation $\rho$, to an objective that is otherwise LeJEPA's. It keeps the advantages of LeJEPA: a two-term objective trained end to end, no decoder, teacher, or negatives, and a target whose single parameter has a clear meaning as the share of shared content. On public audio-visual benchmarks, it is competitive with open-source methods under both frozen probing and fine-tuning.}

{
    \small
    \bibliographystyle{ieeenat_fullname}
    \bibliography{avlejepa}
}

\clearpage
\onecolumn
\appendix
\setcounter{equation}{0}\renewcommand{\theequation}{A\arabic{equation}}\renewcommand{\theHequation}{A\arabic{equation}}

\section{Proofs}

\subsection{Trace representation of the weighted Fisher information (Eq.~2)}
\label{app:proof_jw}

The scalar identity $(\nabla\eta(\mathbf{q})^{\top}\nabla\log p_z(\mathbf{q}))^2=\nabla\log p_z(\mathbf{q})^{\top}\nabla\eta(\mathbf{q})\nabla\eta(\mathbf{q})^{\top}\nabla\log p_z(\mathbf{q})$ and $\mathbb{E}_{\eta}[\nabla\eta(\mathbf{q})\nabla\eta(\mathbf{q})^{\top}]=W$ give $\mathbb{E}_{\eta}[(\nabla\eta(\mathbf{q})^{\top}\nabla\log p_z(\mathbf{q}))^2]=\nabla\log p_z(\mathbf{q})^{\top}W\nabla\log p_z(\mathbf{q})$. Averaging over $\mathbf{q}\sim p_z$, we rewrite the quadratic form as a trace by the cyclicity of the trace and then move the expectation inside the trace, since the Fisher information is finite:
\begin{equation}
 \begin{aligned}
 J_W(p_z)
 &=\mathbb{E}_{\mathbf{q}}[\nabla\log p_z(\mathbf{q})^{\top}W\nabla\log p_z(\mathbf{q})]\\
 &=\mathbb{E}_{\mathbf{q}}\operatorname{tr}(W\nabla\log p_z(\mathbf{q})\nabla\log p_z(\mathbf{q})^{\top})
 =\operatorname{tr}\!\left(W\mathcal{I}(p_z)\right).
 \end{aligned}
\end{equation}

\subsection{Role of weighted Fisher information in the two probes (Eq.~3)}
\label{app:proof_probes}

Both relations follow the proofs of \citet{lejepa} (lemma and appendix numbers refer to arXiv v3). We restate their task setting, estimators and regularity conditions, and expand only the term that involves the second moment of the task gradient, the one term our prior changes. We write $K$ for the embedding dimension ($K=2d$ for the joint embedding) and $\mathcal{K}$ for the kernel.

Conditional on $\eta$, the training samples $\mathcal D$ are i.i.d.\ and independent of $\mathbf q$, and we define $\operatorname{Bias}(\mathbf q)=\mathbb E_{\mathcal D}[\widehat y(\mathbf q)\mid\eta,\mathbf q]-\eta(\mathbf q)$ and $\mathrm{ISB}:=\mathbb E_{\eta,\mathbf q}[\operatorname{Bias}(\mathbf q)^{2}]$. A prediction with a zero denominator is set to zero; this event is absorbed in the remainder terms below.

\paragraph{Radius-based $k$-NN probe: task and estimator.} Following \citet[Appendix~A.1]{lejepa}, the embedding density $p_z$ is $C^{3}$ with derivatives of order up to three bounded and with finite Fisher information and covariance. The unknown labels are given by a target function $\eta:\mathbb{R}^{K}\to\mathbb{R}$ assumed $C^{2}$; classification is handled by setting $\eta(\mathbf{z})=\mathbb{P}(Y=1\mid\mathbf{z})$. The training data consist of the $N$ embeddings together with their labels $\{(\mathbf{z}_n,\eta(\mathbf{z}_n))\}_{n=1}^{N}$, and we write $y_n:=\eta(\mathbf{z}_n)$. The prediction at a query $\mathbf{q}$ averages the labels of the training points inside a ball of radius $r_0$,
\begin{equation}
 \widehat{y}(\mathbf{q}):=\frac{1}{N_{r_0}(\mathbf{q})}\sum_{n\,:\,\|\mathbf{z}_n-\mathbf{q}\|\le r_0}y_n,
\end{equation}
\begin{equation}
 N_{r_0}(\mathbf{q}):=\#\left\{n:\|\mathbf{z}_n-\mathbf{q}\|\le r_0\right\},
\end{equation}
where the radius $r_0$ controls how many neighbors are averaged. Query points follow the training distribution, $\mathbf{q}\sim p_z$.

Following Lemma~4 of \citet{lejepa}, we write the bias expansion as
\begin{equation}
 \operatorname{Bias}(\mathbf{q})
 =\frac{r_0^{2}}{K+2}\left[\nabla\eta(\mathbf{q})^{\top}\nabla\log p_z(\mathbf{q})+\tfrac{1}{2}\Delta\eta(\mathbf{q})\right] +\xi_{r_0}(\mathbf q).
 \label{eq:a_knn_bias}
\end{equation}
The remainder $\xi_{r_0}$ includes finite-sample corrections; we assume $\mathbb E_{\eta,\mathbf q}[\xi_{r_0}^{2}]=o(r_0^{4})$ in the joint large-sample, small-radius limit. Following their notation we abbreviate $A(\mathbf{q}):=\nabla\eta(\mathbf{q})^{\top}\nabla\log p_z(\mathbf{q})$ and $C(\mathbf{q}):=\tfrac12\Delta\eta(\mathbf{q})$. Squaring Eq.~(\ref{eq:a_knn_bias}) and taking the expectation over $\mathbf{q}\sim p_z$ and over the task prior gives their three-term decomposition,
\begin{equation}
 \mathbb{E}_{\eta,\mathbf{q}}\!\left[\operatorname{Bias}(\mathbf{q})^{2}\right]
 =\left(\frac{r_0^{2}}{K+2}\right)^{2}
 \Big\{\underbrace{\mathbb{E}_{\eta,\mathbf{q}}\!\left[A^{2}\right]}_{\text{score--gradient}}
 +2\underbrace{\mathbb{E}_{\eta,\mathbf{q}}\!\left[AC\right]}_{\text{cross}}
 +\underbrace{\mathbb{E}_{\eta,\mathbf{q}}\!\left[C^{2}\right]}_{\text{curvature}}\Big\}+o(r_0^{4}).
 \label{eq:a_knn_three_terms}
\end{equation}

Only the score--gradient term involves the second moment of the task gradient. Expanding this term with the joint prior of Section~3 in place of the isotropic one, and using the identity of Section~\ref{app:proof_jw},
\begin{equation}
 \begin{aligned}
 \mathbb{E}_{\eta,\mathbf{q}}\!\left[A^{2}\right]
 &=\mathbb{E}_{\mathbf{q}}\,\mathbb{E}_{\eta}\!\left[\left(\nabla\eta(\mathbf{q})^{\top}\nabla\log p_z(\mathbf{q})\right)^{2}\right]\\
 &=\mathbb{E}_{\mathbf{q}}\operatorname{tr}\!\left(\nabla\log p_z(\mathbf{q})\nabla\log p_z(\mathbf{q})^{\top}\,
 \mathbb{E}_{\eta}\!\left[\nabla\eta(\mathbf{q})\nabla\eta(\mathbf{q})^{\top}\right]\right)\\
 &=\operatorname{tr}\!\left(W\,\mathbb{E}_{\mathbf{q}}\!\left[\nabla\log p_z(\mathbf{q})\nabla\log p_z(\mathbf{q})^{\top}\right]\right)\\
 &=\operatorname{tr}\!\left(W\mathcal{I}(p_z)\right)\;=\;J_W(p_z).
 \end{aligned}
\end{equation}
Where \citet{lejepa} substitute $\mathbb{E}_{\eta}[\nabla\eta\nabla\eta^{\top}]=\tau_g^{2}I_K$ at this point and obtain $\tau_g^{2}J(p_z)$, the structured prior retains the matrix $W$ and yields $J_W(p_z)$; setting $W=\tau_g^{2}I_K$ recovers their expression exactly.

For the cross and curvature terms we additionally assume $\mathbb{E}_{\eta}[\nabla\eta(\mathbf{q})\,\Delta\eta(\mathbf{q})]=0$ and $\mathbb{E}_{\eta}[(\Delta\eta(\mathbf{q}))^{2}]=\kappa^{2}$, where $\kappa^{2}<\infty$ is independent of $\mathbf{q}$ and $p_z$, and keep the decomposition of \citet[Appendix~B.4]{lejepa}. The cross term then vanishes, and the curvature term equals $\tfrac14\mathbb{E}_{\eta}[(\Delta\eta)^{2}]=\kappa^{2}/4$, which does not depend on $p_z$ and enters at the same order $O(r_0^{4})$ as the $J_W$ term, as a constant offset. Substituting the three terms into Eq.~(\ref{eq:a_knn_three_terms}),
\begin{equation}
 \mathbb{E}_{\eta,\mathbf{q}}\!\left[\operatorname{Bias}(\mathbf{q})^{2}\right]
 =\frac{r_0^{4}}{(K+2)^{2}}\,J_W(p_z)+\frac{r_0^{4}\kappa^{2}}{4(K+2)^{2}}+o(r_0^{4}),
\end{equation}
which is the first relation of Eq.~(3), with $c_{k\mathrm{NN},0}=r_0^{4}\kappa^{2}/(4(K+2)^{2})$ from the curvature term, $c_{k\mathrm{NN},1}=r_0^{4}(K+2)^{-2}$, and $\varepsilon_{k\mathrm{NN}}=o(r_0^{4})$ from the remainder (the cross term vanishes). As in \citet{lejepa}, the relation identifies $J_W(p_z)$ as the part of the integrated squared bias that depends on the shape of $p_z$ at leading order.

\paragraph{Nadaraya--Watson kernel probe: task and estimator.} Following \citet[Appendix~A.2]{lejepa}, consider a nonnegative kernel $\mathcal{K}:\mathbb{R}^{K}\to\mathbb{R}$ with the standard properties
\begin{align}
 \int_{\mathbb{R}^{K}}\mathcal{K}(\mathbf{u})\,\mathrm{d}\mathbf{u}&=1, &&\text{(normalized)}\\
 \int_{\mathbb{R}^{K}}\mathbf{u}\,\mathcal{K}(\mathbf{u})\,\mathrm{d}\mathbf{u}&=0, &&{\text{(zero first moment)}}\\
 \int_{\mathbb{R}^{K}}\mathbf{u}\mathbf{u}^{\top}\mathcal{K}(\mathbf{u})\,\mathrm{d}\mathbf{u}&=\mu_2(\mathcal{K})I_K, &&{\text{(isotropic second moment)}}\\
 \int_{\mathbb{R}^{K}}\mathcal{K}(\mathbf{u})^{2}\,\mathrm{d}\mathbf{u}&<\infty, &&\text{(finite roughness)}
\end{align}
for some $\mu_2(\mathcal{K})\in(0,\infty)$. For a bandwidth $h>0$ write $\mathcal{K}_h(\mathbf{t}):=h^{-K}\mathcal{K}(\mathbf{t}/h)$. The Nadaraya--Watson estimator at a query $\mathbf{q}$ is
\begin{equation}
 \widehat{y}(\mathbf{q}):=\frac{\sum_{n=1}^{N}\mathcal{K}_h(\mathbf{q}-\mathbf{z}_n)\,y_n}
 {\sum_{n=1}^{N}\mathcal{K}_h(\mathbf{q}-\mathbf{z}_n)}.
\end{equation}

As for the $k$-NN probe, the labels are $y_n=\eta(\mathbf{z}_n)$, so both probes estimate the same task function. The regularity conditions of the $k$-NN derivation are retained, and $p_z$ is additionally assumed to have light enough tails that $p_z(\mathbf{z})\to0$ and $z_j\,p_z(\mathbf{z})\to0$ as $\|\mathbf{z}\|\to\infty$ for every coordinate $j$. The task prior is supported on the smoothness class
\begin{equation}
 \mathcal{M}(L,B):=\left\{\eta\in C^{2}(\mathbb{R}^{K}):
 \left\|\nabla\eta(\mathbf{q})\right\|\le L,\ \left|\Delta\eta(\mathbf{q})\right|\le B,\ \forall\mathbf{q}\in\mathbb{R}^{K}\right\}.
 \label{eq:a_kernel_class}
\end{equation}
Following Lemma~5 of \citet{lejepa}, for $h\to0$ and $Nh^K\to\infty$ we write
\begin{equation}
 \operatorname{Bias}(\mathbf{q})
 =\frac{h^{2}\mu_2(\mathcal{K})}{2}\left[\Delta\eta(\mathbf{q})+2\nabla\eta(\mathbf{q})^{\top}\nabla\log p_z(\mathbf{q})\right]{+\xi_h(\mathbf q).}
 \label{eq:a_kernel_bias}
\end{equation}
The remainder $\xi_h$ includes finite-sample corrections; we assume $\mathbb E_{\eta,\mathbf q}[\xi_h^{2}]=o(h^{4})$ in the joint large-sample, small-bandwidth limit.

We now follow the steps of \citet[Appendix~B.7]{lejepa}. Squaring Eq.~(\ref{eq:a_kernel_bias}), averaging over tasks and queries $\mathbf q\sim p_z$ and applying $(a+b)^{2}\le2a^{2}+2b^{2}$ pointwise,
\begin{equation}
 \begin{aligned}
 {\mathbb{E}_{\eta,\mathbf{q}}}\!\left[\operatorname{Bias}(\mathbf{q})^{2}\right]
 &=\left(\frac{h^{2}\mu_2(\mathcal{K})}{2}\right)^{2}
 {\mathbb E_{\eta}}\int\left[\Delta\eta+2\nabla\eta^{\top}\nabla\log p_z\right]^{2}p_z+o(h^{4})\\
 &\le\left(\frac{h^{2}\mu_2(\mathcal{K})}{2}\right)^{2}
 \left[2{\mathbb E_{\eta}}\int\left(\Delta\eta\right)^{2}p_z
 +8{\mathbb E_{\eta}}\int\left(\nabla\eta^{\top}\nabla\log p_z\right)^{2}p_z\right]+o(h^{4}).
 \end{aligned}
 \label{eq:a_kernel_split}
\end{equation}
The first integral is bounded exactly as in LeJEPA's proof: the curvature bound of Eq.~(\ref{eq:a_kernel_class}) gives $\int(\Delta\eta)^{2}p_z\le B^{2}$.

We modify only the treatment of the second integral. \citet{lejepa} bound it with the worst-case gradient norm of the class $\mathcal{M}(L,B)$, applying Cauchy--Schwarz and $\|\nabla\eta\|\le L$ to obtain $L^{2}J(p_z)$, which does not retain the directional structure of the task prior. We instead take the task expectation first, so that the task-moment identity in Section~\ref{app:proof_jw} applies pointwise and the integral is evaluated rather than bounded,
\begin{equation}
 \begin{aligned}
 \int\mathbb{E}_{\eta}\!\left[\left(\nabla\eta(\mathbf{q})^{\top}\nabla\log p_z(\mathbf{q})\right)^{2}\right]p_z(\mathbf{q})\,\mathrm{d}\mathbf{q}
 &=\int \nabla\log p_z(\mathbf{q})^{\top}W\nabla\log p_z(\mathbf{q})\,p_z(\mathbf{q})\,\mathrm{d}\mathbf{q}\\
 &=\operatorname{tr}\!\left(W\mathcal{I}(p_z)\right)\;=\;J_W(p_z).
 \end{aligned}
\end{equation}

Substituting both integrals into Eq.~(\ref{eq:a_kernel_split}),
\begin{equation}
 \mathbb{E}_{\eta,\mathbf{q}}\!\left[\operatorname{Bias}(\mathbf{q})^{2}\right]
 \le\left(\frac{h^{2}\mu_2(\mathcal{K})}{2}\right)^{2}\left[2B^{2}+8J_W(p_z)\right]+o(h^{4}),
\end{equation}
which is the second relation of Eq.~(3), with
\begin{equation}
 c_{\mathrm{ker},0}=2B^{2}\left(\frac{h^{2}\mu_2(\mathcal{K})}{2}\right)^{2},
 \qquad
 c_{\mathrm{ker},1}=8\left(\frac{h^{2}\mu_2(\mathcal{K})}{2}\right)^{2},
 \qquad
 \varepsilon_{\mathrm{ker}}=o(h^{4}).
\end{equation}
These task-averaged bias relations motivate the use of $J_W(p_z)$ as the distribution-design criterion in Section~3.

\begingroup
\subsection{Sources of the isotropic Fisher bound and equality conditions (Eq.~5)}
\label{app:proof_isotropic}
The unweighted Fisher bound and Gaussian equality condition used in Eq.~(5) are established in Lemma~6 of \citet[Appendix~B.4, p.~30]{lejepa}; the trace-constraint step and its equality condition are given in Step~2(a) of the same appendix. The resulting isotropic Gaussian optimum $p_z^{*}=\mathcal{N}(0,I_K)$ is the trace-constrained case of that appendix's concluding optimality statement.
\par
\endgroup

\subsection{Proof of Proposition~1 (Eqs.~12 and~13)}
\label{app:proof_moment}

\begin{proof}[Proof of Proposition~1]
By the independent-information conditions of Assumption~2, $W_{\mathrm{ind}}$ has diagonal blocks $\tau_{\mathrm{ind}}^{2}I_d$ and zero cross blocks. For the shared-content component, Eq.~(9) gives identical gradients on the two sides, so all four blocks of $W_{\mathrm{same}}$ equal $\tau_{\mathrm{same}}^{2}I_d$. This proves Eq.~(12). Component independence and zero means give $\mathbb{E}_{\eta}[\nabla\eta_{\mathrm{ind}}\nabla\eta_{\mathrm{same}}^{\top}]=0$, and expanding Eq.~(7) yields
\[
 \begin{aligned}
 \mathbb{E}_{\eta}[\nabla\eta\nabla\eta^{\top}]
 &=(1-\pi)W_{\mathrm{ind}}+\pi W_{\mathrm{same}}\\
 &\quad+\sqrt{\pi(1-\pi)}\,
 \mathbb{E}_{\eta}[\nabla\eta_{\mathrm{ind}}\nabla\eta_{\mathrm{same}}^{\top}
 +\nabla\eta_{\mathrm{same}}\nabla\eta_{\mathrm{ind}}^{\top}]\\
 &=(1-\pi)W_{\mathrm{ind}}+\pi W_{\mathrm{same}}.
 \end{aligned}
\]
The weights $(1-\pi)$ and $\pi$ in Eq.~(13) are thus the squared component amplitudes of Eq.~(7). The combined task prior also satisfies $\mathbb{E}_{\eta}[\nabla\eta]=0$. Since $\tau_{\mathrm{ind}}^{2}>0$ and $\pi<1$, we have $1-\omega>0$, so the two eigenvalues $\tau_{\mathrm{AV}}^{2}(1+\omega)$ and $\tau_{\mathrm{AV}}^{2}(1-\omega)$ of $W_{\mathrm{AV}}$ (Eq.~(\ref{eq:m_UWU})) are both positive{, \textit{i.e.}\ $W_{\mathrm{AV}}\succ0$}.
\end{proof}

\subsection{Proof of Proposition~2 (Eq.~14)}
\label{app:proof_cr}

\begin{proof}[Proof of Proposition~2]
Let $p_z$ satisfy the conditions of Proposition~2, namely zero mean{,} covariance $\Sigma=\mathbb{E}_{\mathbf{z}\sim p_z}[\mathbf{z}\mathbf{z}^{\top}]\succ0$, $p_z\in C^{3}$ with bounded derivatives up to order three, finite Fisher information and covariance, and $p_z$ everywhere positive on $\mathbb{R}^{2d}$. Under these conditions, the proof of Lemma~6 of LeJEPA \citep[Appendix~B.4]{lejepa} applies the matrix Cram\'er--Rao bound~\citep{cramer1946,rao1945} to the family of shifted densities $p_{\bm{\mu}}(\mathbf{z})=p_z(\mathbf{z}-\bm{\mu})$, whose Fisher information with respect to the shift $\bm{\mu}$ is $\mathcal{I}(p_z)$. It yields $\mathcal{I}(p_z)\succeq\Sigma^{-1}$, with equality if and only if $p_z=\mathcal{N}(0,\Sigma)$.

Multiplying $\mathcal{I}(p_z)-\Sigma^{-1}\succeq0$ by $W_{\mathrm{AV}}^{1/2}$ on both sides and taking traces,
\begin{equation}
 J_{W_{\mathrm{AV}}}(p_z)-\operatorname{tr}\!\left(W_{\mathrm{AV}}\Sigma^{-1}\right)
 =\operatorname{tr}\!\left(W_{\mathrm{AV}}^{1/2}\left(\mathcal{I}(p_z)-\Sigma^{-1}\right)W_{\mathrm{AV}}^{1/2}\right)\ \ge\ 0,
\end{equation}
which is Eq.~(14). If Eq.~(14) holds with equality, the positive semidefinite matrix $W_{\mathrm{AV}}^{1/2}\left(\mathcal{I}(p_z)-\Sigma^{-1}\right)W_{\mathrm{AV}}^{1/2}$ has zero trace and is therefore $0$. Since $W_{\mathrm{AV}}\succ0$, the square root $W_{\mathrm{AV}}^{1/2}$ is invertible, so $\mathcal{I}(p_z)=\Sigma^{-1}$ and the equality condition above gives $p_z=\mathcal{N}(0,\Sigma)$. Conversely, if $p_z=\mathcal{N}(0,\Sigma)$ then $\mathcal{I}(p_z)=\Sigma^{-1}$ and Eq.~(14) holds with equality.
\end{proof}

\subsection{Proof of Proposition~3 (Eq.~15)}
\label{app:proof_target}

\begin{proof}[Proof of Proposition~3]
The proof has two steps.
\emph{Step 1 (a lower bound from the Cauchy--Schwarz inequality).} In the Frobenius inner product $\langle A,B\rangle_F=\operatorname{tr}(A^{\top}B)$, take $A=W_{\mathrm{AV}}^{1/2}\Sigma^{-1/2}$ and $B=\Sigma^{1/2}$. By the cyclic property of the trace, $\langle A,B\rangle_F=\operatorname{tr}(\Sigma^{-1/2}W_{\mathrm{AV}}^{1/2}\Sigma^{1/2})=\operatorname{tr}(W_{\mathrm{AV}}^{1/2})$. We also have $\|A\|_F^{2}=\operatorname{tr}(W_{\mathrm{AV}}\Sigma^{-1})$ and $\|B\|_F^{2}=\operatorname{tr}(\Sigma)=2d$ by the trace constraint, so that
\begin{equation}
 \left[\operatorname{tr}\!\left(W_{\mathrm{AV}}^{1/2}\right)\right]^{2}
 \le\operatorname{tr}\!\left(W_{\mathrm{AV}}\Sigma^{-1}\right)\operatorname{tr}(\Sigma),
 \qquad\text{that is}\qquad
 \operatorname{tr}\!\left(W_{\mathrm{AV}}\Sigma^{-1}\right)
 \ \ge\ \frac{\left[\operatorname{tr}\!\left(W_{\mathrm{AV}}^{1/2}\right)\right]^{2}}{2d}.
\end{equation}
Equality holds if and only if $A$ and $B$ are proportional, \textit{i.e.}\ $W_{\mathrm{AV}}^{1/2}\Sigma^{-1/2}=c\,\Sigma^{1/2}$, that is $\Sigma\propto W_{\mathrm{AV}}^{1/2}$; the trace constraint $\operatorname{tr}(\Sigma)=2d$ then fixes the constant,
\begin{equation}
 \Sigma^{*}=\frac{2d}{\operatorname{tr}\!\left(W_{\mathrm{AV}}^{1/2}\right)}\,W_{\mathrm{AV}}^{1/2}.
 \label{eq:m_sigma_prop}
\end{equation}

\emph{Step 2 (computing $W_{\mathrm{AV}}^{1/2}$ and normalizing).} We first diagonalize $W_{\mathrm{AV}}$: let
\begin{equation}
 U=\frac{1}{\sqrt2}\begin{bmatrix}I_d&I_d\\ I_d&-I_d\end{bmatrix},
 \qquad U^{\top}U=I_{2d},
\end{equation}
which replaces $\mathbf{z}=[\mathbf{e}_a;\mathbf{e}_v]$ by $\left[(\mathbf{e}_a+\mathbf{e}_v)/\sqrt2\,;\,(\mathbf{e}_a-\mathbf{e}_v)/\sqrt2\right]$. A direct computation gives
\begin{equation}
 U^{\top}W_{\mathrm{AV}}U
 =\tau_{\mathrm{AV}}^{2}\begin{bmatrix}(1+\omega)I_d&0\\ 0&(1-\omega)I_d\end{bmatrix},
 \label{eq:m_UWU}
\end{equation}
so $U$ diagonalizes $W_{\mathrm{AV}}$ into two $d$-dimensional blocks. Taking the square root in this basis and transforming back to the original basis,
\begin{equation}
 \begin{aligned}
 W_{\mathrm{AV}}^{1/2}
 &=U\,\tau_{\mathrm{AV}}\begin{bmatrix}\sqrt{1+\omega}\,I_d&0\\ 0&\sqrt{1-\omega}\,I_d\end{bmatrix}U^{\top}\\[2pt]
 &=\frac{\tau_{\mathrm{AV}}}{2}
 \begin{bmatrix}
 \left(\sqrt{1+\omega}+\sqrt{1-\omega}\right)I_d & \left(\sqrt{1+\omega}-\sqrt{1-\omega}\right)I_d\\
 \left(\sqrt{1+\omega}-\sqrt{1-\omega}\right)I_d & \left(\sqrt{1+\omega}+\sqrt{1-\omega}\right)I_d
 \end{bmatrix}.
 \end{aligned}
\end{equation}
This can be checked by squaring: the diagonal blocks give $\tau_{\mathrm{AV}}^{2}[(1+\omega)+(1-\omega)]/2=\tau_{\mathrm{AV}}^{2}$ and the off-diagonal blocks give $\tau_{\mathrm{AV}}^{2}[(1+\omega)-(1-\omega)]/2=\tau_{\mathrm{AV}}^{2}\omega$. Its trace is $\operatorname{tr}(W_{\mathrm{AV}}^{1/2})=d\,\tau_{\mathrm{AV}}\left(\sqrt{1+\omega}+\sqrt{1-\omega}\right)$, so $\tau_{\mathrm{AV}}$ and $d$ cancel in Eq.~(\ref{eq:m_sigma_prop}). Multiplying the numerator and denominator of $\rho$ by $\sqrt{1+\omega}+\sqrt{1-\omega}$ then gives
\begin{equation}
 \Sigma^{*}(\rho)=\begin{bmatrix}I_d&\rho I_d\\ \rho I_d&I_d\end{bmatrix},
 \qquad
 \rho=\frac{\sqrt{1+\omega}-\sqrt{1-\omega}}{\sqrt{1+\omega}+\sqrt{1-\omega}}
 =\frac{\omega}{1+\sqrt{1-\omega^{2}}}.
 \label{eq:m_rho_pf}
\end{equation}
Both diagonal blocks $\Sigma^{*}_{aa},\Sigma^{*}_{vv}$ of Eq.~(\ref{eq:m_rho_pf}) are $I_d$, so the two modalities have equal total variance $d$ at the optimum without any per-modality constraint.
\end{proof}

\end{document}